\pdfoutput=1
\documentclass[11pt]{article}

\usepackage{emnlp2022}

\usepackage{times}
\usepackage{latexsym}

\usepackage[T1]{fontenc}
\usepackage[utf8]{inputenc}

\usepackage{microtype}

\usepackage{inconsolata}

\usepackage{amsmath}
\usepackage{amsfonts}
\usepackage{amssymb}
\usepackage{bbm}
\usepackage{amsthm}
\usepackage{booktabs}
\usepackage{multirow}

\usepackage{paralist}

\usepackage{mdwlist}
\usepackage{graphicx}
\usepackage[inkscapelatex=false]{svg}
\usepackage{color}

\usepackage[ruled,linesnumbered]{algorithm2e}

\usepackage{xspace}
\usepackage{makecell}
\newtheorem{theorem}{Theorem}
\newtheorem{lemma}[theorem]{Lemma}

\newtheorem{innercustomthm}{Theorem}
\newenvironment{customthm}[1]
  {\renewcommand\theinnercustomthm{#1}\innercustomthm}
  {\endinnercustomthm}

\def\longmethod{\textbf{D}istillation-\textbf{R}esistant \textbf{W}atermarking\xspace}
\def\method{DRW\xspace}

\title{Distillation-Resistant Watermarking for Model Protection in NLP}

\author{Xuandong Zhao ~~~ Lei Li  ~~~ Yu-Xiang Wang \\
  University of California, Santa Barbara\\
  \texttt{\{xuandongzhao,leili,yuxiangw\}@cs.ucsb.edu} \\
}

\begin{document}
\maketitle

\begin{abstract}
% The first sentence: define the problem. 
% Second sentence: existing methods have limitations or define challenge to solve the problem. 
% Third sentence: In this paper, we propose \method, a novel method to xxxx. 
% Fourth sentence: key insight or intuitive or novelty of the method. 
% Fifth sentence: We evaluate \method on xxx public datasets. 
% Experiments show that our proposed \method achives the state of the art with xxx score on ....

How can we protect the intellectual property of trained NLP models? Modern NLP models are prone to stealing by querying and distilling from their publicly exposed APIs. However, existing protection methods such as watermarking only work for images but are not applicable to text. We propose \longmethod(\method), a novel technique to protect NLP models from being stolen via distillation. \method protects a model by injecting watermarks into the victim's prediction probability corresponding to a secret key and is able to detect such a key by probing a suspect model. We prove that a protected model still retains the original accuracy within a certain bound. We evaluate \method on a diverse set of NLP tasks including text classification, part-of-speech tagging, and named entity recognition. Experiments show that \method protects the original model and detects stealing suspects at 100\% mean average precision for all four tasks while the prior method fails on two\footnote{Our code is available at \url{https://github.com/XuandongZhao/DRW}}.   

%Training modern NLP models can be costly, demanding a massive amount of training data, powerful computing resources, and human expertise. Nevertheless, profitable NLP models are also likely to be stolen via model extraction, where an adversary utilizes knowledge distillation to train a surrogate model based on the results generated by a prediction API of the original model. In this paper, we propose \method, a novel watermarking technique against model extraction attacks. By injecting designed noise into the output of the original model, we are able to verify whether a suspect model steals the victim model. We evaluate \method on text classification and token classification tasks. The empirical results corroborate that the defender can claim model ownership with 100\% mean average precision and that \method outperforms the state-of-the-art baselines. 

% NLP APIs have become essential  services in many commercial companies. Protecting NLP models from model stealing attack remains a challenging problem for the discrete nature of the text representation.

\end{abstract}

\section{Introduction}
\label{sec:intro}
Large-scale pre-trained  neural models have shown great success in NLP tasks \cite{devlin-etal-2019-bert, liu2019roberta}. 
Task-specific NLP models are often deployed as web services with pay-per-query APIs in business applications.
Protecting the intellectual property of these cloud deployed models is a critical issue in both research and practice. 
Service providers often use authentication mechanism to authorize valid accesses. 
However, while this prevents clients directly copying a victim model, it does not hinder clients from stealing it using distillation. 
Emerging model extraction attacks have demonstrated convincingly that most functions of the victim API are likely to be stolen with carefully designed queries \cite{Tramr2016StealingML, Wallace2020ImitationAA, krishna2020thieves, He2021ModelEA}. 
A model extraction process is often imperceptible because it queries APIs in the same way as a normal user does \cite{Orekondy2019KnockoffNS}.
In this paper, we study the problem of \emph{model protection} for NLP against distillation stealing. 

%Training large-scale NLP models is quite expensive, which requires not only large-scale datasets but also massive computational resources. The training cost can grow rapidly with task complexity and model capacity. For instance, it can cost \$1.6 million to train a BERT model on Wikipedia and Book corpora (15 GB) \cite{Sharir2020TheCO}. Due to the underlying commercial value, intellectual property (IP) protection for Machine-learning-as-a-service (MLaaS) models has attracted increasing interest from both academia and  industry.

Little has been done to adapt watermarking to identify model infringements in language tasks.
Although a number of defense techniques have been proposed to prevent the model extraction for computer vision, they are not applicable to language tasks with discrete tokens. 
Among them, deep neural networks (DNN) watermarking \cite{Szyller2021DAWNDA,jia2021entangled} works by embedding a secret watermark (e.g., logo or signature) into the model exploiting the over-parameterization property of DNNs. 
This procedure leverages a trigger set to stamp invisible watermarks on their commercial models before distributing them to customers. 
When suspicion of model theft arises, model owners can conduct an official ownership claim with the aid of the trigger set. 
However, these protections all focus on the image/audio tasks, since it is easy to modify the continuous data. 
In addition, most watermarking methods are invasive and fragile. They cannot avoid tampering with the training procedure in order to embed the watermark. Besides, the watermarks are outliers of the task distribution so that the adversary may not carry the watermark through distillation. 

To fill in the gap, we make the first attempt to protect NLP models from distillation. 
We propose \longmethod(\method) to protect models and detect suspicious stealing. 
Inspired by the idea from CosWM for computer vision \cite{charette2022cosine}, 
we utilize prediction perturbation to embed a secret sinusoidal signal to the output of the victim API. 
To handle discrete tokens, we design a technique to randomly project tokens to a uniform region within sinusoidal cycles. 
We design watermarking effective for distillation with soft labels and with hard-sampled labels. 
As long as the adversary trains the distillation procedure till convergence, \method is able to detect the watermark signal from the extracted model. 

The advantages of \method include 1) \emph{training independence}: it works directly on the trained models and can be directly plugged into the final output. 2) \emph{flexibility}: it can be applied to both soft-label output and hard-label output in the black-box setting. 3) \emph{effectiveness}: we evaluate the effectiveness of \method and obtain perfect model extraction detection accuracy; we also justify the fidelity with a negligible side effect on the original classification quality. 4) \emph{scalability}: the secret keys for the watermark are randomly generated on the fly so that we are able to provide different watermarks for different end-users and verify them.

The contributions of this paper are as follows:
\begin{itemize*} 
    \item We enhance the concept of model protection against model extraction attacks with an emphasis on language applications.
    \item We propose \method, a novel method to inject watermarks to the output of the NLP models and later to detect if suspects distill from the victim.
    \item We provide a theoretical guarantee on the protected API accuracy --- with protection \method does not harm much of original API's performance.  
    \item Experiments on four diverse tasks (POS Tagging/NER/SST-2/MRPC) verify that \method detects extracted models with 100\% mean average precision, yet with only a small drop (<5\%) in original prediction performance.
\end{itemize*}

\section{Related Work}
\label{sec:related}
\noindent\textbf{Model Extraction Attacks}~
Model extraction attacks target the confidentiality of ML models and aim to imitate the function of a black-box victim model \cite{Tramr2016StealingML, Orekondy2019KnockoffNS, correia2018copycat}. First, adversaries collect or synthesize an initially unlabeled substitute dataset. Next, they exploit the ability to query the victim model APIs for label predictions to annotate the substitute dataset. Then, they can train a high-performance model utilizing the pseudo-labeled dataset. Recently, several works \cite{krishna2020thieves, Wallace2020ImitationAA, He2021ModelEA} attempt to address the model extraction attacks on NLP models, e.g. BERT \cite{devlin-etal-2019-bert} or Google Translate.

\noindent\textbf{Knowledge Distillation}~
Model extraction attacks are closely related to knowledge distillation (KD) \cite{Hinton2015DistillingTK}, where the adversary acts as the student who approximates the behaviors of the teacher (victim) model. The student can learn from soft labels or hard labels. KD with soft labels has been widely applied due to the fact that soft labels can carry a lot of useful information \cite{phuong2019towards, Zhou2021RethinkingSL}.

\noindent\textbf{Watermarking}~
A digital watermark is an undetected label embedded in a noise-tolerant signal, such as audio, video, or image data. It is designed to identify the owner of the signal's copyright. Some works \cite{Uchida2017EmbeddingWI, Adi2018TurningYW, Zhang2018ProtectingIP, Merrer2019AdversarialFS} employ watermarks to prevent precise duplication of machine learning models. They insert watermarks into the parameters of the protected model or construct backdoor images that activate particular predictions. If an adversary exactly copies a protected model, a watermark can be used to verify ownership. However, safeguarding models from model extraction attacks is more difficult due to the fact that the parameters of the suspect model might be vastly different from those of the victim model, and the backdoor behavior may not be transferred to the suspect model either. Several works \cite{Juuti2019PRADAPA, Szyller2021DAWNDA, jia2021entangled, charette2022cosine, He2021ProtectingIP} study how to identify extracted models that are distilled from the victim model. \citet{jia2021entangled} forces the protected model to acquire features for identifying data samples taken from authentic and watermarked data. \citet{He2021ProtectingIP} conducts lexical modification as a watermarking method to protect language generation APIs. CosWM \cite{charette2022cosine} incorporates a watermark as a cosine signal into the output of the protected model. Since the cosine signal is difficult to eliminate, extracted models trained via distillation will continue to have a significant watermark signal. Nonetheless, CosWM only applies to image data and soft distillation. We design multiple new techniques to extend CosWM in handling the text data with discrete sequence and we provide a theoretical guarantee on the protected API accuracy for soft and hard distillations

\section{Proposed Method: \method}
\label{sec:approach}
\begin{figure*}[t]
\centering
% \includesvg[width=1.0\linewidth]{fig/model.svg}
\includegraphics[width=1.0\textwidth]{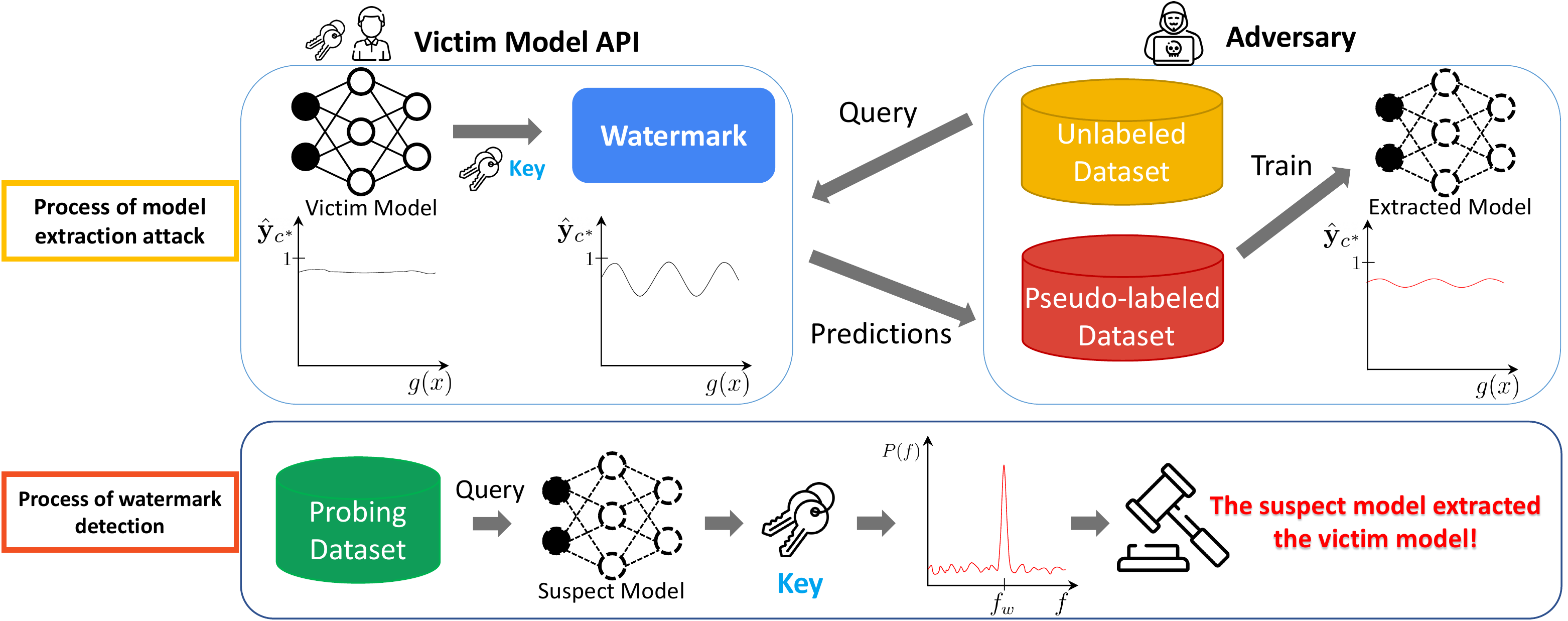}
\caption{Overview of model extraction attack and watermark detection. The upper panel illustrates that the API owner adds a sinusoidal perturbation to the predicted probability distribution before answering end-users. The extracted model will convey this periodical signal if the adversary distills the victim model. At the phase of watermark detection, as shown in the bottom panel, the owner queries the suspect model and applies Fourier transform to the output with a key. Then the designed perturbation can be detected when a peak shows up in the frequency domain at $f_w$. The extracted watermark can thus serve as legal evidence and judgment for the ownership claim.}
\label{fig:overview}
\vspace{-4mm}
\end{figure*}

\subsection{Overview}
Figure \ref{fig:overview} presents an overview of distillation procedure, watermarking and detection. The main idea of \method is to introduce a perturbation to the output of a protected model. This designed perturbation is transferred onto a suspect model distilled from a victim model that remains identifiable by probing the suspect model. 

\noindent\textbf{Problem Formulation}~ We consider a common real-world scenario that the adversary only has black-box access to the victim model's API $\mathcal{V}$. There exist two types of output from victim model API: soft (real-valued) labels (i.e. probabilities) and hard labels. The adversary employs an auxiliary unlabeled dataset to query $\mathcal{V}$. Once the adversary gains the predictions from the victim model, it can train a separate model $\mathcal{S}$ from scratch with the pseudo-labeled dataset. The adversary may either distill the victim model with hard labels by minimizing the cross-entropy loss
\begin{equation}
\mathcal{L}_{\text{CE}}=-\sum_{i=1}^{m} \hat{\mathbf{y}}_{i} \log \left(\hat{\mathbf{q}}_{i}\right),
\end{equation}
where $\hat{\mathbf{q}}_i$ is the prediction from the stealer's model and $\hat{\mathbf{y}}$ are the pseudo-labels from the victim model; or distill from soft labels by minimizing the Kullback–Leibler (KL) divergence loss
\begin{equation}
\mathcal{L}_{\text{KL}} = \sum_{i=1}^{m} \hat{\mathbf{y}}_{i} \log \left(\frac{\hat{\mathbf{y}}_{i}}{\hat{\mathbf{q}}_{i}}\right).
\end{equation} 

% The extracted model is considered as a successful stealing if its performance matches the victim model. 

\subsection{Watermarking the Victim Models}
\method dynamically embeds a watermark in response to queries made by an API's end-user. We use a set of variables to represent key $K = (c^*, f_w, \mathbf{v}_k, \mathbf{v}_s, \mathbf{M})$, where $c^*\in \{1, \dots, m\}$ is the target class to embed watermark; $f_w \in \mathbb{R}$ is the angular frequency; $\mathbf{v}_k \in \mathbb{R}^n$ is the phase vector; $ \mathbf{v}_s \in \mathbb{R}^n$ is the selection vector; $\mathbf{M}\in \mathbb{R}^{|D|\times n}$ is the random token matrix. $|D|$ represents the vocabulary size, so that every token ID corresponds to vector $\mathbf{M}_i \in \mathbb{R}^n$.  Following \citet{charette2022cosine}, we define a periodic signal function based on $K$ and the input $x$.
\begin{equation}
\mathbf{z}_{c}(x)=\left\{\begin{array}{lr}
\cos \left(f_{w} g(\mathbf{v}_k, x, \mathbf{M})\right), &  c=c^{*} \\
\cos \left(f_{w} g(\mathbf{v}_k, x, \mathbf{M})+\pi\right), & c\neq c^{*}
\end{array}\right.
\end{equation}
for $c \in \{1, \dots, m\}$, where $g(\cdot)\in [0, 1)$ is a hash function projecting a text representation to a scalar. Ideally, the scalar should uniformly distribute spanning multiple cycles.

\paragraph{Constructing the hash function} We project every input $x$ into the fixed scalar range to add the sinusoidal perturbation by the hash function $g(\cdot)$. We randomly generate the phase vector $\mathbf{v}_k$, selection vector $\mathbf{v}_s$ and the token matrix $\mathbf{M}$. Each element in $\{\mathbf{v}_k, \mathbf{v}_s\}$ is randomly sampled from a uniform distribution over $[0, 1)$. Each element of the matrix $\mathbf{M}$ is randomly sampled from a standard normal distribution $\mathbf{M}_{ij} \sim \mathcal{N}(0, 1)$. Let $\mathbf{M}_{i} \in \mathbb{R}^n$ denote the $i$-th row of matrix $\mathbf{M}$, $\mathbf{v}_k^\top\mathbf{M}_{i} \sim \mathcal{N}(0, \frac{n}{3})$ and $\mathbf{v}_s^\top\mathbf{M}_{i} \sim \mathcal{N}(0, \frac{n}{3})$ (we prove it in Appendix \ref{lemma1}). Then we apply probability integral transformation to obtain the uniform distribution of the hash values, where $g(\mathbf{v}_k, x, \mathbf{M})\sim \mathcal{U}(0, 1)$ and $g(\mathbf{v}_s, x, \mathbf{M})\sim \mathcal{U}(0, 1)$. We set $g(\mathbf{v}_s, x, \mathbf{M}) \leq \tau$ to select part of all samples, where $\tau$ is the data selection ratio. When implementing sequence labeling tasks, we use the token ID to fetch the vector in matrix $\mathbf{M}$. Similarly, when implementing sentence classification tasks, we use the ID of the second token in the sentence to obtain the vector.

Next we compute the periodic signal for the victim output
\begin{equation}
\label{eq:softmax}
\small
\setlength{\tabcolsep}{2pt} 
\hat{\mathbf{y}}_c = 
\left\{\begin{array}{lr}
\hat{\mathbf{p}}_{c}, & g(\mathbf{v}_s, x , \mathbf{M}) > \tau \\
\frac{\hat{\mathbf{p}}_{c}+\varepsilon\left(1+\mathbf{z}_{c}(x)\right)}{1+2 \varepsilon}, &  c=c^{*} \text{~and~} g(\mathbf{v}_s, x, \mathbf{M}) \leq \tau  \\ \frac{\hat{\mathbf{p}}_{c}+\frac{\varepsilon\left(1+\mathbf{z}_{c}(x)\right)}{m-1}}{1+2 \varepsilon}, & c\neq c^{*} \text{~and~} g(\mathbf{v}_s, x, \mathbf{M}) \leq \tau 
\end{array}\right.
\end{equation}
where $\varepsilon$ is the watermark level for the periodic signal and $\hat{\mathbf{p}}_{c}$ is the victim model's prediction before watermarking. Since $0 \leq \hat{\mathbf{y}}_i\leq 1$ and $\sum_{i=1}^{m} \hat{\mathbf{y}}_i = 1$ (see proof in Appendix \ref{lemma2}), $\hat{\mathbf{y}}$ is a surrogate for softmax output. 

In the soft label setting, the victim model generates output $\hat{\mathbf{y}}$ directly; while in the hard label setting, the victim model produces the sampling hard label, i.e. a one-hot label with probability $\hat{\mathbf{y}}_i$ for each class $i$. Intuitively, the hard-label sampled output retains the watermark because it is equal to $\hat{\mathbf{y}}$ in \emph{expectation}. Further, we define the accuracy for soft label output, named ``argmax soft'', which calculates the accuracy of the argmax of soft output compared with the true label. Similarly, we define ``sampling hard'' to describe the output of the victim model which is a one-hot vector.

% \begin{algorithm}[tb]
% \SetAlgoLined
% \SetKwInOut{Input}{Input}
% \Input{Dataset $D$ (after tokenization / splitting), labelling policies $\pi,\pi_c$, number of epochs $T$ }
% $D' \leftarrow \textrm{Dedup}(D)$\\
% $D'' \leftarrow \textrm{Redact}_\pi(D')$\\
%  $D^{pri}\leftarrow \{ s \in D'' | \exists x\in s \text{ s.t. } \pi(s, x)=1 \text{ or } \exists x\subset s \text{ s.t. }\pi_c(s, x) = 1  \}$ \\
%  $D^{pub} \leftarrow  \{s\in D'' | s\notin D^{pri}\}$\\ 
% \For{$e = 1,...,T$}
%  {
%  Run one epoch of SGD with $D^{pub}$\\
%  Run one epoch of DP-SGD with $D^{pri}$
%  }
% \caption{Embedding Watermarks to the Victim Model}\label{alg:alg_1}
% \end{algorithm}

\subsection{Detecting Watermark from Suspect Models}
We first create a probing dataset $\mathcal{D}_p$, for which the labels are not required. $\mathcal{D}_p$ can be drawn from the training data of the extracted model since the owner is able to store any query sent by a specific end-user. In our setting, we also allow $\mathcal{D}_p$ to be drawn from other distributions.

We employ the Lomb-Scargle periodogram method \cite{Scargle1982StudiesIA} for detecting and characterizing periodic signals. The Lomb–Scargle periodogram yields an estimate of the Fourier power spectrum $P(f)$ at frequency $f$ in an unevenly sampled dataset. After getting the power spectrum, we evaluate the signal strength by calculating the signal-to-noise ratio
\begin{gather}
P_{\text {signal }}=\frac{1}{\delta} \int_{f_{w}-\frac{\delta}{2}}^{f_{w}+\frac{\delta}{2}} P(f) df \notag \\ 
\text{\footnotesize $P_{\text {noise }}=\frac{1}{F-\delta}\left[\int_{0}^{f_{w}-\frac{\delta}{2}} P(f) d f+\int_{f_{w}+\frac{\delta}{2}}^{F} P(f) d f\right]$} \notag \\
P_{\text {snr}}=P_{\text {signal }} / P_{\text {noise }},
\end{gather}

where $\delta$ controls the window width of $\left[f_{w}-\frac{\delta}{2}, f_{w}+\frac{\delta}{2}\right]$; $F$ is the maximum frequency, and $f_{w}$ is the angular frequency embedded into the victim model.
A higher signal-to-noise ratio $P_{\text {snr}}$ indicates a higher peak in the frequency domain.

\section{Theoretical Analysis}
\label{sec:exps}
%\citet{charette2022cosine} has already provided theoretical bounds on the strengths of the watermark. 
In this section, we provide theoretical guarantees for \method for both argmax soft output and sampling hard output. The analysis assumes the victim is \emph{calibrated} so its soft-predictions are informative. We also focus on the binary classification task, i.e., $m=2$. Generalization to $m>2$ is straightforward and omitted only to ensure a clean presentation.
\begin{customthm}{1} \label{thm:acc}
Without loss of generality, set target class $c^* = 1$, so that $\hat{p} = \hat{\mathbf{p}}_1(x), \hat{y} = \hat{\mathbf{y}}_1, z(x) = \mathbf{z}_1(x)$.
Assume $\hat{p}(x)$ is calibrated, i.e., $\mathbb{E}[y|\hat{p}(x)=a]=a, ~\forall 0\leq a\leq 1$, the argmax soft label of the victim model is $\hat{y}_s = \mathbbm{1}\{\frac{\hat{p}(x)+\varepsilon(1+z(x))}{1+2\varepsilon} > 0.5\}$ and the sampling hard label of the victim output is $\hat{y}_h \sim \operatorname{Ber}(\frac{\hat{p}(x)+\varepsilon(1+z(x))}{1+2\varepsilon})$. For a fixed $\mathbf{v}_k$, given that 
$z(x) = \cos \left(f_{w} g(\mathbf{v}_k, x, \mathbf{M})\right)\in [-1, 1]$ and the data selection ratio is set to $\tau$, then 
\method argmax soft label and sampling hard label satisfy: 
\begin{align}
    \mathbb{E}_{\mathbf{v}_k}&\left[ \text{Acc}(\text{Argmax Soft}) \right] \geq \text{Acc}(\text{Victim})\notag\\
    & - \tau(0.5+\varepsilon)\mathbb{P}[0.5-\varepsilon\leq \hat{p}\leq 0.5+\varepsilon],\label{eq:acc_soft}\\
    %\notag\\
    \mathbb{E}_{\mathbf{v}_k}&\left[\text{Acc}(\text{Sampling Hard})\right] \geq (1-\tau)\text{Acc}(\text{Victim})\notag\\
    & + \frac{\tau}{1+2\varepsilon} \mathbb{E}\left[2\hat{p}^2 -2\hat{p} + 1\right].\label{eq:acc_hard}
\end{align}
\end{customthm}
The proof is deferred to Appendix \ref{thm:A1}.
% \begin{proof}
% See Appendix \ref{thm:A1} for a proof.
% \end{proof}
% \begin{theorem}\label{thm:soft}
% With target class $c^* = 1$, set $\hat{p} = \hat{\mathbf{p}}_1(x), \hat{y} = \hat{\mathbf{y}}_1, z(x) = \mathbf{z}_1(x)$.
% Assume $\hat{p}(x)$ is calibrated, i.e., $\mathbb{E}[\mathbb{P}(y=1|x)|\hat{p}(x)=a]=a, ~\forall 0\leq a\leq 1$, and $\hat{y} = \mathbbm{1}\{\frac{\hat{p}(x)+\varepsilon(1+z(x))}{1+2\varepsilon} > 0.5\}$. Given that 
% $z(x) = \cos \left(f_{w} g(\mathbf{v}_k, x, \mathbf{M})\right) \in [-1, 1]$,
% \method argmax soft label has the following bounds with a fixed $\mathbf{v}_k$: 
% \begin{align*}
%     \text{Acc}(&\text{Argmax Soft}) \geq \text{Acc}(\text{Victim Model})\\
%     & - (0.5+\varepsilon)\mathbb{P}(0.5-\varepsilon\leq \hat{p}\leq 0.5+\varepsilon)
% \end{align*}
% \end{theorem}
%Observe that the accuracy gracefully drops from the accuracy of the victim model as $\epsilon$ and $\tau$ get larger. In the soft level setting, only those data points where $\hat{p}$ lies around 0.5 ($\pm \varepsilon$) could be affected --- data points where the model prediction is uncertain anyways.

Equation \eqref{eq:acc_soft} says that, in the soft label setting, \method does not hurt the accuracy too much if the watermark level $\varepsilon$ is small. Note that only samples in which the victim model output lies around 0.5 ($\pm \varepsilon$) might be affected by the watermarking. These are data points where the victim model is uncertain and inaccurate anyway.
%the level of accuracy. 

% We then show the theorem about the accuracy of sampling hard output of the victim model in binary classification task. 
% \begin{theorem}\label{thm:hard} With target class $c^* = 1$, set $\hat{p}(x) = \hat{\mathbf{p}}_1(x), \hat{y} = \hat{\mathbf{y}}_1, z(x) = \mathbf{z}_1(x)$.
% Assume $\hat{p}(x)$ is calibrated, i.e., $\mathbb{E}[y|\hat{p}(x)=a]=a, ~\forall 0\leq a\leq 1$,  and hard label $\hat{y} \sim \operatorname{Ber}(\frac{\hat{p}(x)+\varepsilon(1+z(x))}{1+2\varepsilon})$. Given that 
% % $\mathbf{z}(x) = \operatorname{cos}\left(2\pi f_w F^{-1}\left(\frac{v^\top\phi(x)}{\lVert \phi(x) \rVert}\right)\right)$.
% $z(x) = \cos \left(f_{w} g(\mathbf{v}_k, x, \mathbf{M})\right)$,  and $g$ is the distribution function, which is subject to $g(\mathbf{v}_k, x, \mathbf{M}) \sim \mathcal{U}(0, 1)$, 
% \method sampling hard label has the following bounds with a fixed $\mathbf{v}_k$: \\
% $$\text{Acc(Sampling Hard)} \geq \frac{1}{1+2\varepsilon} \mathbb{E}\left[2\hat{p}^2 -2\hat{p} + 1\right]$$
% \end{theorem}

Equation \eqref{eq:acc_hard} lowerbounds the accuracy of the sampled hard labels, which is close to the vanilla victim model if $\tau$ is small. Observe that if $\tau = 1$, the accuracy may drop even if the watermark magnitude $\varepsilon=0$ due to the sampling of the output label\footnote{Under the calibration assumption,
$Acc(Victim)= \mathbb{E}\left[\hat{p}\mathbf{1}(\hat{p}\geq0.5) + (1-\hat{p})\mathbf{1}(\hat{p}<0.5)\right]$, which is strictly bigger than $\mathbb{E}\left[2\hat{p}^2 -2\hat{p} + 1\right]$ except when $\hat{p}$ is supported only at trivial points $\{0,1,0.5\}$}. Our design of a second random projection $\mathbf{v}_s$ plays an important role here as it allows us to control the accuracy drop to any level we desire by adjusting $\tau$.

%Note that the accuracy of sampling hard output is closely related to the watermark level $\varepsilon$ and the prediction value $\hat{p}$ before adding watermark. Consider a high performance victim model with $\hat{p}$ close to 1, the accuracy of sampling hard can be very high if we choose a small watermark level $\varepsilon$.

\section{Experiments}
\label{sec:exps}
\begin{table}[t]
\centering
\setlength{\tabcolsep}{4pt} 
\begin{tabular}{lcccc}
\Xhline{2\arrayrulewidth} 
Model Type & SST-2 & MRPC & POS & NER \\
\hline
\multicolumn{5}{l}{mAP of detection for soft distillation:}\\
% \multirow{2}{*}{\shortstack{mAP of soft \\ distillation}} 
% & DeepJudge* & 0.206 & 0.206 & 0.254  &  0.280   \\
~~~DeepJudge* & 1.00 & 1.00 & 0.54  &  0.84 \\
~~~\method & 1.00 & 1.00 & 1.00  & 1.00  \\
\hline
\multicolumn{5}{l}{mAP of detection for hard distillation:}\\
% \multirow{2}{*}{\shortstack{mAP of hard \\ distillation}} 
% & DeepJudge* & 0.206 & 0.206 &  0.269   &  0.216  \\
~~~DeepJudge* & 1.00 & 1.00 &  0.48   &  0.40  \\
~~~\method & 1.00 & 1.00  & 1.00  & 1.00  \\
\hline 
\multicolumn{5}{l}{Performance of the models:}\\
% \multirow{5}{*}{\shortstack{Performance of \\ the models}} & 
~~BERT & 92.9 & 86.7 & - & 92.4 \\
~~Victim model & 92.8 & 87.0 & 90.7 & 91.3 \\
~~~~+argmax soft & 92.5 & 86.8 & 90.7 & 91.3 \\
~~~~+sampling hard & 88.4 & 85.8 & 90.3 & 91.0 \\

~~Adversary soft & 92.0 & 86.2 & 89.8 & 87.7 \\
~~Adversary hard & 91.3 & 86.1 & 89.7 & 87.4 \\
\Xhline{2\arrayrulewidth} 
\end{tabular}
\caption{Main results for detection and model performance. We report the mean average precision of the model infringements detection for both soft-label distillation and hard-label distillation. The baseline is constructed based on the modification of DeepJudge. We show the results for BERT reported in the original paper. We report the results of victim model for argmax soft and sampling hard. }
\vspace{-4mm}
\label{table:main}
\end{table}

\begin{table*}[t]
\centering
\begin{tabular}{lccccc}
\Xhline{2\arrayrulewidth} 
& SST-2 & MRPC & POS & NER \\
\hline
DeepJudge-$JSD$-Soft: \\
~~~Negative Suspect & (0.012, 0.032) & (0.009, 0.161) & (0.016, 0.444)  &  (0.001, 0.416)  \\
~~~Positive Suspect & (0.001, 0.002) & (0.001, 0.002) & (0.087, 0.279)  & (0.002, 0.201) \\
\hline 
DeepJudge-$JSD$-Hard: \\
~~~Negative Suspect & (0.013, 0.029) & (0.008, 0.154) & (0.010, 0.432)  &  (0.009, 0.274)   \\
~~~Positive Suspect & (0.004, 0.005) & (0.003, 0.007) & (0.029, 0.112) & (0.011, 0.052) \\
\hline
\method-$P_{\text{snr}}$-Soft: \\
~~~Negative Suspect & (0.008, 4.775) & (0.128, 2.607) & (0.012, 2.309) & (0.105, 4.243)   \\
~~~Positive Suspect & (18.82, 25.77) & (17.81, 24.25) & (20.59, 28.73) & (17.25, 25.22) \\
\hline
\method-$P_{\text{snr}}$-Hard: \\
~~~Negative Suspect & (0.011, 4.235) & (0.012, 3.678) & (0.182, 2.869)  &  (0.203, 4.183) \\
~~~Positive Suspect & (16.38, 22.77) & (16.70, 21.80) & (16.23, 25.67) &  (16.19, 25.49)\\
\Xhline{2\arrayrulewidth} 
\end{tabular}
\caption{The probing results for DeepJudge and \method in soft distillation and hard distillation settings. We present the range of $JSD$ and $P_{\text{snr}}$. The first value in parentheses is the minimum score and the second value is the maximum score. A larger gap in score between the negative and positive suspect models indicates that the detection method performs better in identifying the extracted model.}
% \vspace{-4mm}
\label{table:new}
\end{table*}

\subsection{Tasks}\label{sec:task}
We evaluate the performance of \method on four different tasks. Two are sequence labeling tasks, Part-Of-Speech (POS) Tagging and Named Entity Recognition (NER); the other two are from GLUE \cite{Wang2018GLUEAM} text classification tasks, SST-2 and MRPC. We choose BERT \cite{devlin-etal-2019-bert} as our model backbone and fine-tune it in different tasks.

\paragraph{Sequence labeling} We utilize the CoNLL-2003 dataset \cite{Sang2003IntroductionTT} for POS Tagging and NER tasks. The CoNLL-2003 dataset consists of news articles from the Reuters
RCV1 corpus with POS and NER tags. We formulate POS Tagging and NER as token-level classification tasks following standard practice. Specifically, POS Tagging has 47 classes and NER has 9 classes. We take the token embedding of the last hidden layer of BERT \cite{devlin-etal-2019-bert} as the input to a linear layer, which is then used as the classifier over the POS/NER label set. The token ID is set as the input $x$ for the hash function $g(\cdot)$. F1 score is hired for the evaluation metric. 

\paragraph{Text classification} SST-2 is a binary single-sentence classification task consisting of movie reviews with corresponding sentiment \cite{Socher2013RecursiveDM}. MRPC is a collection of sentence pairs from online news with labels suggesting whether the pair is semantically equivalent or not \cite{Dolan2005AutomaticallyCA}. We use the final hidden vector of the special \texttt{[CLS]} token of BERT as the input to a linear layer, which serves as the sentence classifier. The ID of the second token in the sentence is set as the input $x$ for the hash function $g(\cdot)$. Since GLUE does not include any test dataset, we use accuracy of the validation set as the evaluation metric.

For each task, we train the protected model to achieve the best performance on the validation set. As demonstrated in Table \ref{table:main}, the victim model has comparable performance to BERT \cite{devlin-etal-2019-bert}. For soft and hard label distillation, we split the training data in each task into two parts and use the first half to query the victim model. Then the extracted model is trained for 20 epochs on the pseudo-labeled dataset. We choose the same key $K = (c^*, f_w, \mathbf{v}_k, \mathbf{v}_s, \mathbf{M})$, where frequency $f_w = 16.0$, watermark level $\varepsilon = 0.2$ and $\{\mathbf{v}_k, \mathbf{v}_s, \mathbf{M}\}$ are generated with different random seed. We set target class $c^*=22$ (``NNP'' tag) for POS Tagging, $c^*=2$ (``I-PER'' tag) for NER and $c^*=0$ (``negative''  class) for SST-2/MRPC. We set data selection ratio $\tau = 0.5$ to add watermarks to half of the output data. More details for the experiment setting can be found in Appendix \ref{sec:details}.

\begin{figure*}[h]
\centering
\includegraphics[width=1.0\textwidth]{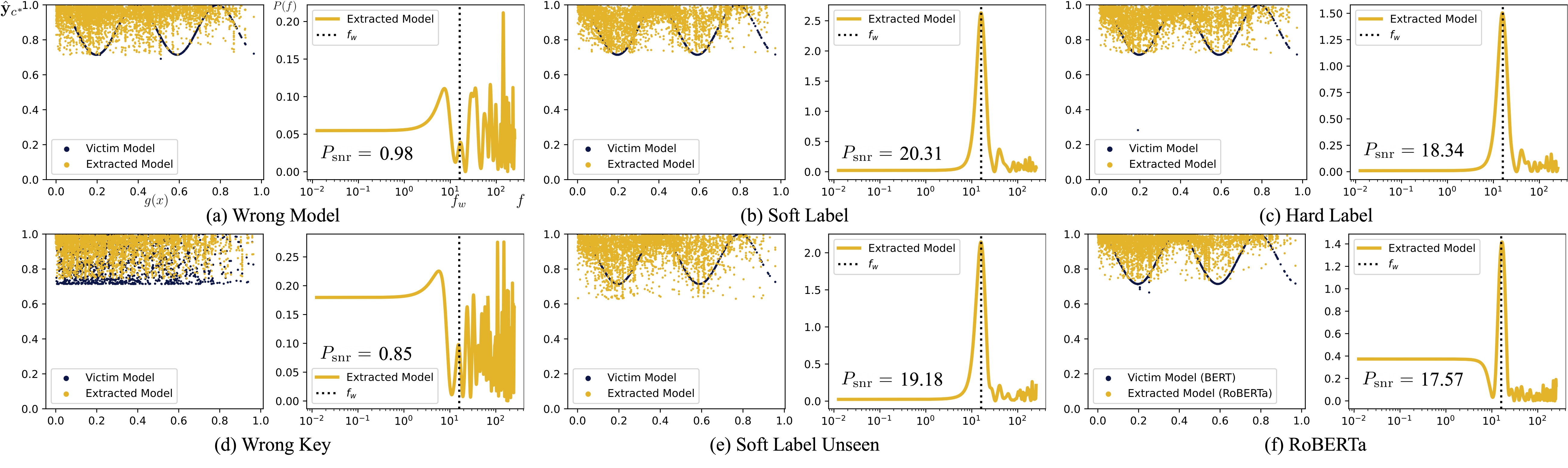}
% \includesvg[width=1.0\linewidth]{fig/emnlp_case.svg}
\caption{Examples of \method in NER task. The left panel of each sub-figure plots the output of the target class $c^*$ for the victim model and the extracted model ($\hat{\mathbf{y}}_{c^*} \text{~vs.~} g(x)$). The right panel of each sub-figure plots the power spectrum value for output of the extracted model($P(f) \text{~vs.~} f$). We also display the $P_{\text{snr}}$ value for signal strength of the extracted model.}
\vspace{-5mm}
\label{fig:case}
\end{figure*}

\paragraph{Baseline} We take the state-of-the-art method DeepJudge \cite{chen2022copy} as a baseline against \method. DeepJudge quantitatively tests the similarities between the victim model and suspect model, then  determines whether the suspect model is a copy based on the testing metrics. Since DeepJudge is designed for continuous signals such as images and audio, we modify the method to apply it to texts. We consider the black-box setting for DeepJudge, and compute Jensen-Shanon Distance ($JSD$) \cite{Fuglede2004JensenShannonDA} for the probing dataset of the victim model and the extracted model. $JSD$ measures the similarity of two probability distributions. We use the probing dataset to query both the victim model and the suspect model, and then calculate $JSD$ of the output layer as follows
$$
\text{\small $JSD\left(\hat{\mathbf{y}}, \hat{\mathbf{q}}\right)=\frac{1}{2|\mathcal{D}_p|} \sum_{x \in \mathcal{D}_p} K\left(\hat{\mathbf{y}}(x), u\right)+K\left(\hat{\mathbf{q}}(x), u\right)$}
$$
where $u=\left(\hat{\mathbf{y}}(x)+\hat{\mathbf{p}}(x)\right) / 2$ and $K(\cdot, \cdot)$ is the Kullback-Leibler divergence. A small $JSD$ value implies similar output distribution of the two models, which further indicates that the suspect model may be distilled from the victim model.

\paragraph{Evaluation} We evaluate the performance of the victim model and the extracted model with accuracy/F1 score. In order to compare \method with DeepJudge in detecting extracted models, we can reduce this binary classification problem to thresholding a particular test score. Since \method and DeepJudge use different scores to detect the extracted model, we set up a series of ranking tasks to show the effect of these scores. For each task, we train 10 extracted models from the watermarked victim model with different random initialization as positive samples, 10 extracted models from the unwatermarked victim model with different random initialization, and 10 models from scratch with true labels as negative samples. For \method, we use the watermark signal strength values $P_{\text{snr}}$ as the score for ranking (identifying whether it is an extracted model); for DeepJudge, we use $JSD$ as the score. Next, we compute the mean average precision (mAP) for the ranking tasks which assesses the model extraction detection performance. A higher mAP means the detecting method can distinguish the victim and the suspect model better.

We show the experiment results in the following subsections.

\subsection{Effectiveness: Is \method able to identify model infringements?}
We evaluate our method in two settings, distillation with soft labels and distillation with hard labels. The results are displayed in Table \ref{table:main}. DeepJudge performs well on SST-2 and MRPC tasks but it can not effectively detect the extracted models in POS Tagging and NER tasks. In contrast, our method can successfully detect the extraction with 100\% mAP across all tasks in both settings. We also present the range of $JSD$ and $P_{\text{snr}}$ in Table \ref{table:new}. Regarding the performance of DeepJudge on POS Tagging and NER tasks, the $JSD$ intervals for positive and negative samples overlap each other, resulting in the aforementioned lower mAP compared to \method. A case in point is DeepJudge-$JSD$-Hard for NER task, where the ranges for negative suspect score and positive suspect score are $[0.009, 0.274]$ and $[0.011, 0.052]$ respectively. The overlapping intervals lead to the imperfect detection result, i.e., mAP $=0.40$. Whereas, \method is able to \emph{perfectly} distinguish between positive and negative suspects. Typically, $P_{\text{snr}}$ for the negative suspect is smaller than 5 while that for the positive suspect is larger than 15.

\subsection{Fidelity: Does \method decrease the performance of the model?}
The results for the model performance at the watermark level $\varepsilon=0.2$ are displayed in Table \ref{table:main}. The perturbed API (victim model with argmax soft/sampling hard) only has a slight performance drop (within 5\%) in comparison to the original one due to the trade-off between detection effectiveness and model performance. For the victim model API, argmax soft exhibits less performance drop than the sampling hard, since the argmax of the soft label remains unchanged with small perturbation. For the extracted model, distillation with soft label tends to have a better accuracy/F1 score than that with hard label. Additionally, the performances of extracted models are very close to those of victim models, a clear manifestation of the distillation success. 

\subsection{Case Study} \label{sec:casestudy}
We present how our method works on some examples in NER task. We fix the victim model and choose different settings for the suspect model across all the examples. For the watermarked ones, we set $f_w = 16, \varepsilon = 0.2 \text{~and~} c^* = 2$. 

In Figure \ref{fig:case} (a), we show how \method works on a suspect model that does not extract the victim model. We select a model trained from scratch with true labels as a negative example. There is no sinusoidal signal in the output of the suspect model hence a small $P_\text{snr}$.

In Figure \ref{fig:case} (b)(c), we illustrate the effect on soft distillation and hard distillation. We use the watermark key $K$ to extract the output of the victim model and suspect model. The extracted model clearly follows the victim model and there is a prominent peak at frequency $f_w$. Note that suspect model distillation with soft labels has a higher $P_{\text{snr}}$ than the one with hard labels. This is because the training process of extracted models can be more effective and faster with soft labels \cite{phuong2019towards}.

In Figure \ref{fig:case} (d), we validate the \emph{secrecy} of our method. If the adversary does not have the secret key, it can not justify what the watermark is or whether there exist watermarks. The output of the victim model and extracted model are almost indiscernible when we use a wrong key to project them given the hash function $g(\cdot)$.

In Figure \ref{fig:case} (e)(f), we demonstrate the generality of our method. Watermarking algorithm should be independent of the dataset and the ML algorithms. In sub-figure (e), a different dataset is used to probe the suspect model. To be specific, we select the second half of the training data as the probing dataset, rather than the first half used in previous experiments. The results imply that \method turns out to work well when we use unseen data to produce the probing dataset for the suspect model. In sub-figure (f), we choose a different backbone RoBERTa \cite{liu2019roberta} for the suspect model, in which the victim model continues to be the BERT model. The high peak in the power spectrum at frequency $f_w$ reveals that \method is still able to detect the signal.

\section{Ablation Study}

\subsection{Does watermark level impact detection?}\label{sec:wmlevel}
\begin{figure}[h]
\centering
\includegraphics[width=0.45\textwidth]{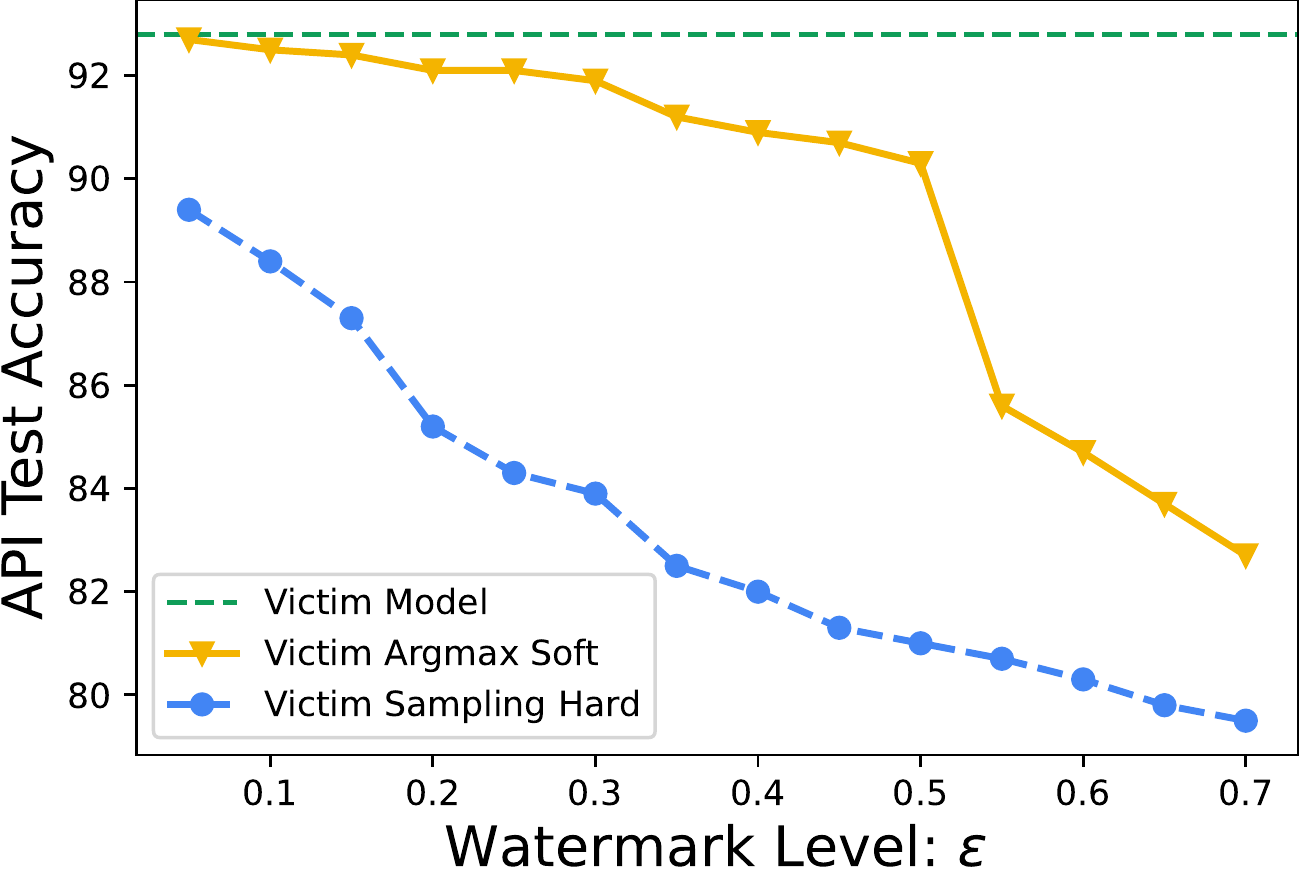}
\caption{Test accuracy of victim model API with different watermark level in SST-2 task. }
\vspace{-2mm}
\label{fig:ab_noise1}
\end{figure}
An important aspect of watermarking is how much perturbation we add to the output of the victim model. Theoretically, a smaller watermark level is associated with a higher accuracy/F1 score of the victim, yet it makes it harder to extract the signal from the probing results. We conduct two experiments to investigate the effect of the watermark level. 

In the first experiment, we vary the watermark level in SST-2 task. According to the Theorem \ref{thm:acc}, the accuracy of the victim model output is bounded and a higher watermark level causes poorer performance.  As shown in Figure \ref{fig:ab_noise1},  when the watermark level rises from 0 to 0.7, the performance drops by around 10 percent. It is worth noting that a big drop of the argmax soft emerges as $\varepsilon$ passes 0.5, which means the argmax of the output is highly likely to be changed in this case.

\begin{figure}[h]
\centering
\includegraphics[width=0.45\textwidth]{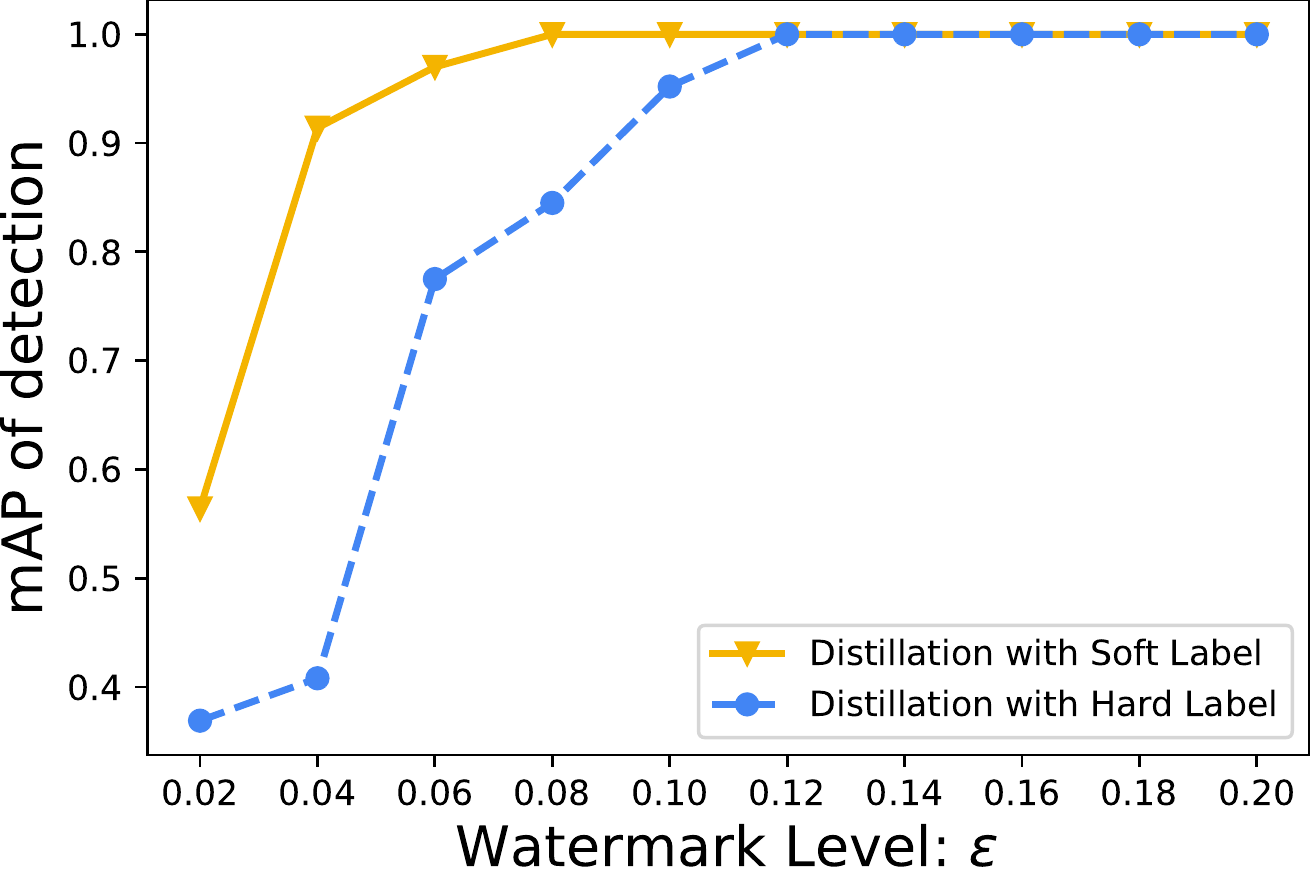}
\caption{Model detection results with different watermark level in NER task.}
\vspace{-2mm}
\label{fig:ab_noise2}
\end{figure}
In the second experiment, we design 10 sets of ranking tasks, and build up 10 positive samples together with 20 negative samples (similar to the setting in Section \ref{sec:task}) for each set in NER task. The watermark level is the only varied parameter across different tasks, ranging from 0.02 to 0.2. We plot the mAP of the detection against the watermark level in Figure \ref{fig:ab_noise2}. When the watermark level $\varepsilon$ is below 0.12, \method can not generate perfect detection of positive and negative suspects, indicating that the adversary may not convey a strong sinusoidal signal at a low watermark level. In this case, \method can not extract the watermark in frequency space and thus fails to detect it successfully. 

These two experiments demonstrate the trade-off between the detection effectiveness and the victim model's performance after watermarking.

\subsection{Do categories affect watermark protection?}
\begin{figure}[h]
\centering
\includegraphics[width=0.45\textwidth]{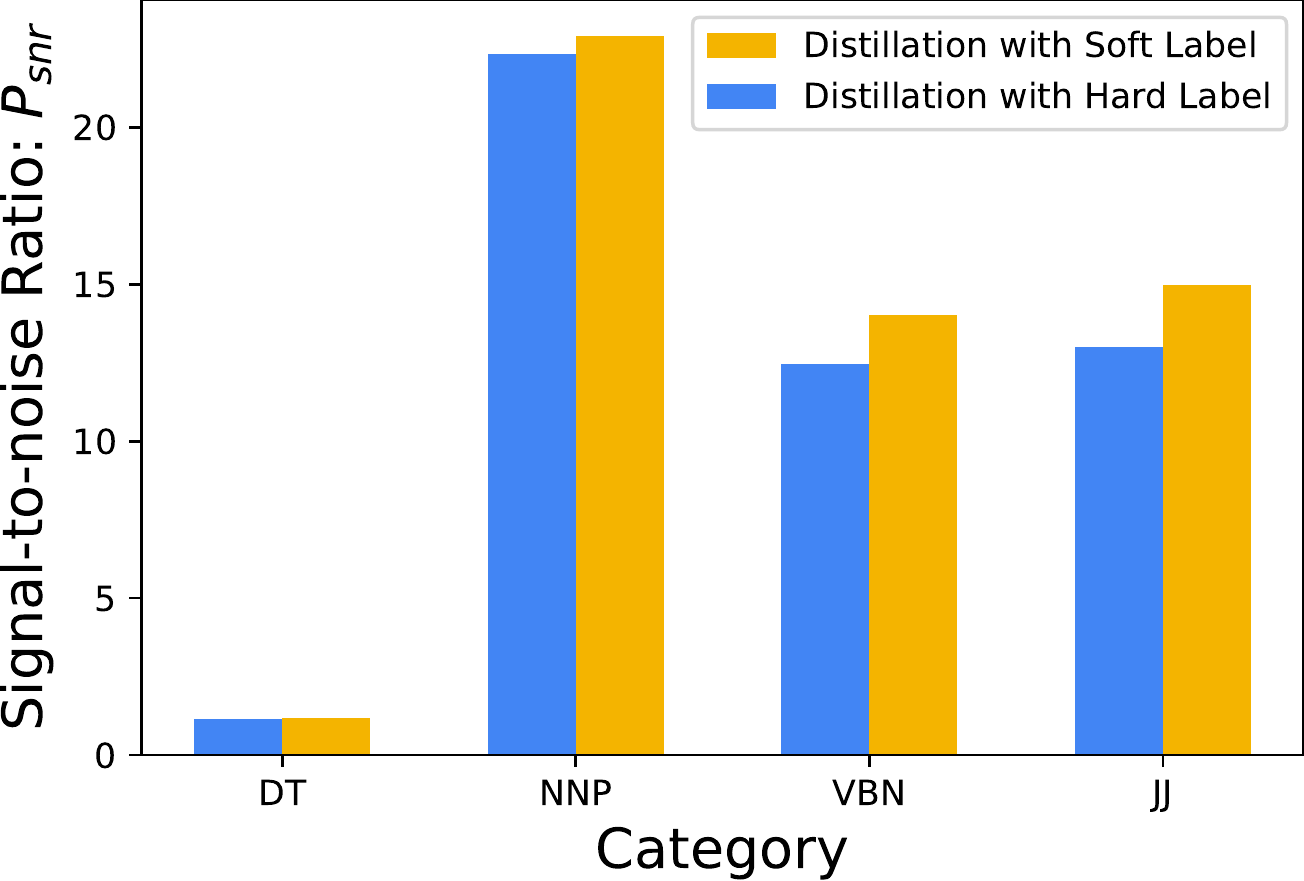}
\caption{Adding watermark to four categories in POS Tagging task. "DT": determiner; "NNP": proper noun, singular; "VBN": verb, past practice; "JJ": adjective.}
\vspace{-2mm}
\label{fig:ab_category}
\end{figure}
We vary the target class $c^*$ of the watermark key $K$ in POS Tagging task. We add watermarks to four different categories and then train the extracted model by soft distillation and hard distillation. The results of the signal-to-noise ratio $P_\text{snr}$ are visualized in Figure \ref{fig:ab_category}. The effect of the watermark will be more salient if the category involves more samples. Since "NNP" covers the most (14.16\%) of all tokens, adding watermark to "NNP" produces the strongest signal. In contrast, the determiners ("DT") category only has a few number of types, such as "the" and "a". As a result, adding watermark to "DT" is ineffective as it is hard to add a periodic signal to a very discrete domain. 

\subsection{How much should be selected for watermarking?}
\begin{figure}[h]
\centering
\includegraphics[width=0.47\textwidth]{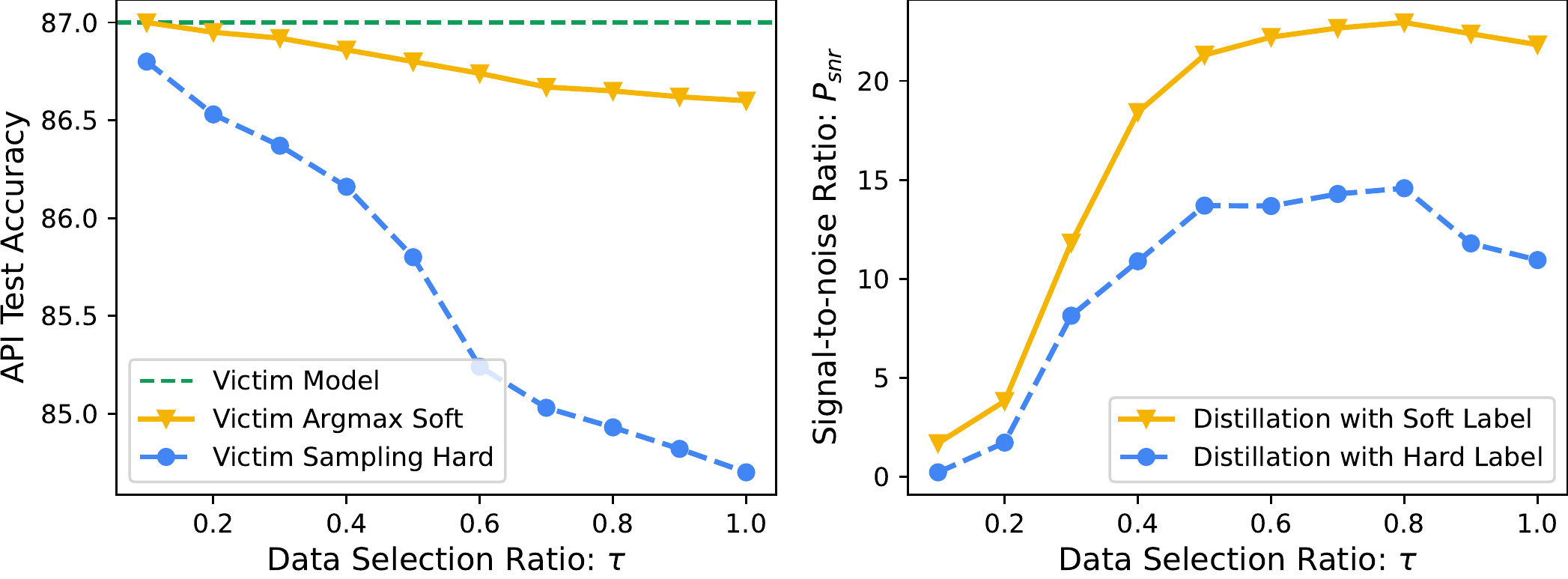}
\caption{Output accuracy of the victim model and signal strength of the extracted model with different data selection ratio $\tau$ in MRPC task.}
\vspace{-2mm}
\label{fig:ab_ratio}
\end{figure}
A critical design of our method is that we apply selection vector $\mathbf{v}_s$ to select a portion of the victim model output to be watermarked. We change the ratio of the watermarked data by tuning the data selection ratio $\tau$ in MRPC task. The results shown in Figure \ref{fig:ab_ratio} indicate that the accuracy of the victim model output falls with a higher data selection ratio, yet it introduces a greater signal strength of the extracted model. This trade-off is similar to the one described in Section \ref{sec:wmlevel}. 0.5 could be a reasonable selection ratio.

% \begin{figure}[htb]
% \centering
% \includegraphics[width=0.30\textwidth]{fig/temp.png}
% \caption{Generality of \method}
% % \vspace{-5mm}
% \label{fig:ab_generality}
% \end{figure}

\section{Conclusion}
\label{sec:conclusion}
In this work, we propose \longmethod(\method), a novel and unified watermarking technique against model extraction attacks on NLP models. By injecting watermarks into the prediction output of the victim model, the model owner can detect the watermark if the adversary distills the protected model.  We prove the theoretical guarantee of \method and show remarkable empirical results on text classification and sequence labeling tasks.

\section*{Limitations}
\label{sec:limit}
1) The watermark detection does not work well when the watermarked data covers only a small amount of the whole training data for the extracted model. 
2) Our method may not work well when the adversary only makes a few queries to the victim model APIs and trains the extracted model with few-shot learning.
3) If the victim model outputs soft labels, even with watermarking, the adversary can take argmax operation to erase the watermark. So it is better to combine watermarks with hard label output in real-world applications.

% EMNLP 2022 requires all submissions to have a section titled ``Limitations'', for discussing the limitations of the paper as a complement to the discussion of strengths in the main text. This section should occur after the conclusion, but before the references. It will not count towards the page limit.  

% The discussion of limitations is mandatory. Papers without a limitation section will be desk-rejected without review.
% ARR-reviewed papers that did not include ``Limitations'' section in their prior submission, should submit a PDF with such a section together with their EMNLP 2022 submission.

% While we are open to different types of limitations, just mentioning that a set of results have been shown for English only probably does not reflect what we expect. 
% Mentioning that the method works mostly for languages with limited morphology, like English, is a much better alternative.
% In addition, limitations such as low scalability to long text, the requirement of large GPU resources, or other things that inspire crucial further investigation are welcome.

\section*{Broader Impact}
\label{sec:impact}
% Scientific work published at EMNLP 2022 must comply with the \href{https://www.aclweb.org/portal/content/acl-code-ethics}{ACL Ethics Policy}. We encourage all authors to include an explicit ethics statement on the broader impact of the work, or other ethical considerations after the conclusion but before the references. The ethics statement will not count toward the page limit (8 pages for long, 4 pages for short papers).
This work will alleviate ethical concerns of commercial NLP models. 
This paper provides one promising solution to an important aspect of NLP: how to protect the intellectual property of trained NLP models. Companies with NLP web services can apply our method to protect their models from model extraction attacks.

\section*{Acknowledgements}
XZ was supported by UCSB Chancellor’s Fellowship. The authors would like to thank Yang Gao for polishing up the draft and Dan Qiao for the helpful discussion.

\bibliography{paper}
\bibliographystyle{acl_natbib}
\newpage
\appendix

\section{Appendix}
\label{sec:appendix}
\subsection{Proof for the Theorem 1}
\begin{customthm}{A.1}[Restate Theorem \ref{thm:acc}]\label{thm:A1}
Without loss of generality, set target class $c^* = 1$, so that $\hat{p} = \hat{\mathbf{p}}_1(x), \hat{y} = \hat{\mathbf{y}}_1, z(x) = \mathbf{z}_1(x)$.
Assume $\hat{p}(x)$ is calibrated, i.e., $\mathbb{E}[y|\hat{p}(x)=a]=a, ~\forall 0\leq a\leq 1$, the argmax soft label of the victim model is $\hat{y}_s = \mathbbm{1}\{\frac{\hat{p}(x)+\varepsilon(1+z(x))}{1+2\varepsilon} > 0.5\}$ and the sampling hard label of the victim output is $\hat{y}_h \sim \operatorname{Ber}(\frac{\hat{p}(x)+\varepsilon(1+z(x))}{1+2\varepsilon})$. For a fixed $\mathbf{v}_k$, given that 
$z(x) = \cos \left(f_{w} g(\mathbf{v}_k, x, \mathbf{M})\right)\in [-1, 1]$ and the data selection ratio is set to $\tau$, then 
\method argmax soft label and sampling hard label satisfy: 
\begin{align}
    \mathbb{E}_{\mathbf{v}_k}&\left[ \text{Acc}(\text{Argmax Soft}) \right] \geq \text{Acc}(\text{Victim})\notag\\
    & - \tau(0.5+\varepsilon)\mathbb{P}[0.5-\varepsilon\leq \hat{p}\leq 0.5+\varepsilon],\\
    %\notag\\
    \mathbb{E}_{\mathbf{v}_k}&\left[\text{Acc}(\text{Sampling Hard})\right] \geq (1-\tau)\text{Acc}(\text{Victim})\notag\\
    & + \frac{\tau}{1+2\varepsilon} \mathbb{E}\left[2\hat{p}^2 -2\hat{p} + 1\right].
\end{align}
\end{customthm}

\begin{proof}
We first prove the argmax soft label case with $\tau=1$.
\begin{align*}
    &\mathbb{E}\left[\mathbbm{1}(\hat{y}_s=y)\right] \\
    =& \mathbb{E}\left[\mathbb{P}(\hat{y}_s=y|x)\right]\\
    =& \mathbb{E}\left[\mathbb{P}(\hat{y}_s=1, y=1|x)+\mathbb{P}(\hat{y}_s=0, y=0|x)\right]\\
    =& \mathbb{E}\left[\mathbb{P}(\hat{y}_s=1|x)\mathbb{P}(y=1|x)\right.\\
    &+\left.\mathbb{P}(\hat{y}_s=0|x)\mathbb{P}(y=0|x)\right]\\
    =& \mathbb{E}\left[\mathbbm{1}\{\hat{p}+\varepsilon z(x)> 0.5\} \mathbb{P}(y=1|x)\right.\\
    &+\left.\mathbbm{1}\{\hat{p}+\varepsilon z(x)\leq0.5\}\left(1-\mathbb{P}(y=1|x)\right)\right]\\
    =& \mathbb{E}\big[\mathbb{E}\left[\mathbbm{1}\{\hat{p}+\varepsilon z(x)> 0.5\} \mathbb{P}(y=1|x)\right.\\
    &+\left.\mathbbm{1}\{\hat{p}+\varepsilon z(x)\leq0.5\}\left(1-\mathbb{P}(y=1|x)\right)|\hat{p}\right]\big]\\
    \geq& \mathbb{E}\big[\mathbb{E}\left[\mathbbm{1}\{\hat{p}-\varepsilon > 0.5\}\mathbb{P}(y=1|x)|\hat{p}\right]\\
    &+ \mathbb{E}\left[\mathbbm{1}\{\hat{p}+\varepsilon \leq 0.5\} (1-\mathbb{P}(y=1|x))|\hat{p}\right]\big]\\
    =& \mathbb{E}\big[\mathbbm{1}\{\hat{p}-\varepsilon > 0.5\}\mathbb{E}\left[\mathbb{P}(y=1|x)|\hat{p}\right]\\
    &+ \mathbbm{1}\{\hat{p}+\varepsilon \leq 0.5\}\mathbb{E}\left[1-\mathbb{P}(y=1|x)|\hat{p}\right]\big]\\
    =& \mathbb{E}\big[\mathbbm{1}\{\hat{p} > 0.5+\varepsilon\}\hat{p} + \mathbbm{1}\{\hat{p} \leq 0.5-\varepsilon\}(1-\hat{p})\big]\\
    =& \underbrace{\mathbb{E}\left[\mathbbm{1}\{\hat{p}>0.5\}\hat{p}+\mathbbm{1}\{\hat{p}\leq0.5\}(1-\hat{p})\right]}_{\text{Accuracy of victim model without watermark}}\\
    &-\mathbb{E}\left[\mathbbm{1}\{0.5<\hat{p}\leq 0.5+\varepsilon\}\hat{p}\right]\\
    &+\mathbb{E}\left[\mathbbm{1}\{0.5-\varepsilon\leq \hat{p}\leq 0.5\}(1-\hat{p})\right]\\
    \geq& \text{Acc}(\text{Victim Model}) \\
    &- (0.5+\varepsilon)\mathbb{P}(0.5-\varepsilon\leq \hat{p}\leq 0.5+\varepsilon)
\end{align*}
where the first "$\geq$" follows from $|z(x)| \leq 1$; the third "$=$" follows from the conditional independence of $\hat{y}_s$ and $y$ given $x$; the seventh "$=$" follows from the calibration assumption, i.e. $\mathbb{E}\left[\mathbb{P}(y=1|x)|\hat{p}(x)\right] = \hat{p}(x)$.

Notice that over the distribution of $\mathbf{v}_s$ selects every unique $x$ with probability $\tau$ independently to everything else, by exchanging the order of expectation, it is easy to prove that the expected accuracy is a convex combination of the accuracy of the victim model (with weight $1-\tau$) and the case above (with weight $\tau$). This completes the proof for argmax soft label.

We then start by analyzing the sampling hard label case with $\tau=1$.
\begin{align*}
    &\mathbb{E}\left[\mathbbm{1}(\hat{y}=y)\right] \\
    =& \mathbb{E}\left[\mathbb{P}(\hat{y}=y|x)\right]\\
    =& \mathbb{E}\left[\mathbb{P}(\hat{y}=1, y=1|x)+\mathbb{P}(\hat{y}=0, y=0|x)\right]\\
    =& \mathbb{E}\left[\mathbb{E}(\hat{y}|x)\mathbb{E}(y|x)+\mathbb{E}(1-\hat{y}|x)\mathbb{E}(1-y|x)\right]\\
    =& \mathbb{E}\left[\left(\frac{\hat{p}}{1+2\varepsilon}+\frac{\varepsilon(1+z(x))}{1+2\varepsilon}\right)\mathbb{E}(y|x)\right. \\
    &+\left. \left(\frac{1-\hat{p}}{1+2\varepsilon}+\frac{\varepsilon(1-z(x))}{1+2\varepsilon}\right)\mathbb{E}(1-y|x)\right]\\
    =& \text{\scriptsize$\frac{1}{1+2\varepsilon}\underbrace{\mathbb{E}\left[\hat{p}\mathbb{E}(y|x)+(1-\hat{p})\mathbb{E}(1-y|x)\right]}_{A}$}\\
    &+ \text{\scriptsize ${\frac{\varepsilon}{1+2\varepsilon} \underbrace{\mathbb{E}\left[(1+z(x))\mathbb{E}(y|x)+(1-z(x))\mathbb{E}(1-y|x)\right]}_{B}}$}
\end{align*}
\begin{align*}
    A =& \mathbb{E}\left[\mathbb{E}\left[\hat{p}\mathbb{E}(y|x)+(1-\hat{p})\mathbb{E}(1-y|x)|\hat{p}\right]\right]\\
    =& \mathbb{E}\left[\hat{p}\mathbb{E}(y|\hat{p})+(1-\hat{p})\mathbb{E}(1-y|\hat{p})\right]\\
    =& \mathbb{E}\left[\hat{p}^2+(1-\hat{p})^2\right]\\
    =& \mathbb{E}\left[2\hat{p}^2-2\hat{p}+1\right]
\end{align*}
where the third line follows from the calibration assumption, i.e., $\mathbb{E}[y|\hat{p}(x)=a]=a$.
\begin{align*}
    B = &\mathbb{E}\left[\mathbb{E}(y|x)+\mathbb{E}(y|x)z(x)+1-z(x)\right.\\
    &-\mathbb{E}(y|x) +\left.\mathbb{E}(y|x)z(x)\right]\\
    = & 1+\mathbb{E}\left[\left(2\mathbb{E}(y|x)-1\right)z(x)\right]\\
    \geq & 0
\end{align*}
where the last line follows from the facts that $|z(x)|\leq 1$ and $|2\mathbb{E}(y|x)-1| \leq 1$.

Finally, notice that for each $x$ the probability to be chosen to add watermark and to sample the output is $\tau$ independently, thus the expected accuracy is the convex combination of the accuracy of the victim model and that of the fully watermarked model.
\end{proof}

\subsection{Distribution Property}\label{lemma1}
\begin{lemma}
Assume $\mathbf{v} \sim \mathcal{U}(0,1),~ \mathbf{v}\in\mathbb{R}^n$ and $\mathbf{x} \sim \mathcal{N}(0,1),~ \mathbf{x}\in \mathbb{R}^n$, where $\mathbf{v}$ and $\mathbf{x}$ are both $i.i.d.$ and independent of each other. Then we have:
$$\frac{1}{\sqrt{n}}\mathbf{v}\cdot \mathbf{x} 	\rightsquigarrow \mathcal{N}\left(0, \frac{1}{3}\right),~ n\rightarrow \infty$$
\end{lemma}

\begin{proof}
Let $u_i=\mathbf{v}_i\mathbf{x}_i,~ i\in{1,2,\ldots,n}$. By assumption, $u_i$ are $i.i.d.$. Clearly, the first and second moments are bounded, so the claim follows from the classical central limit theorem, 
\begin{align*}
    \sqrt{n}\Bar{u}_n &= \frac{\sum_{i=1}^n u_i}{\sqrt{n}} \rightsquigarrow \mathcal{N}\left(\mu, \sigma^2\right)~\text{ as }~n\rightarrow \infty
\end{align*}
where
\begin{align*}
\mu &= \mathbb{E}\left(u_i\right) = \mathbb{E}\left(\mathbf{v}_i\mathbf{x}_i\right) = \mathbb{E}\left(\mathbf{v}_i\right) \mathbb{E}\left(\mathbf{x}_i\right) \\
&= 0 \\
\sigma^2 &= \operatorname{Var}(u_i) = \mathbb{E}\left(u_i^2\right) - \left(\mathbb{E}\left(u_i\right)\right)^2\\
&= \mathbb{E}\left(u_i^2\right) = \mathbb{E}\left(\mathbf{v}_i^2\mathbf{x}_i^2\right) = \mathbb{E}\left(\mathbf{v}_i^2\right)\mathbb{E}\left(\mathbf{x}_i^2\right)\\
&= \frac{1}{3}
\end{align*}
It follows that given large $n$
\begin{align*}
    \frac{1}{\sqrt{n}}\mathbf{v}\cdot \mathbf{x} \rightsquigarrow \mathcal{N}\left(0, \frac{1}{3}\right)
\end{align*}
\end{proof}

\subsection{Modified Softmax Properties}\label{lemma2}
\begin{lemma}[Lemma 1 in \cite{charette2022cosine}]
Let $\hat{\mathbf{p}}$ be the softmax output of a model $\mathcal{V}$, then the modified softmax $\hat{\mathbf{y}}$, as defined in Equation \ref{eq:softmax} satisfies $0 \leq \hat{\mathbf{y}}_i\leq 1$ and $\sum_{i=1}^{m} \hat{\mathbf{y}}_i = 1$.
\begin{proof}
Notice that in Equation \ref{eq:softmax}, when $g(\mathbf{v}_s, x, \mathbf{M}) > \tau$, $\hat{\mathbf{y}}$ = $\hat{\mathbf{p}}$, so that it satisfies the property above.

By the definition of softmax, for all class $c \in \{1, \dots, m\}$ we have
$$
0 \leq \hat{\mathbf{p}}_{c} \leq 1, 
-1 \leq \mathbf{z}_{c}(x) \leq 1 .
$$
Therefore, when $c=c^{*}$, we have
$$
0 \leq \hat{\mathbf{p}}_{c}+\varepsilon\left(1+\mathbf{z}_{c}(x)\right) \leq 1+2 \varepsilon,
$$
and then
$$
0 \leq \frac{\hat{\mathbf{p}}_{c}+\varepsilon\left(1+\mathbf{z}_{c}(x)\right)}{1+2 \varepsilon} \leq 1 .
$$
When $c \neq c^{*}$, since $m \geq 2$, we have
$$
0 \leq \hat{\mathbf{p}}_{c}+\frac{\varepsilon\left(1+\mathbf{z}_{c}(x)\right)}{m-1} \leq 1+\frac{2 \varepsilon}{m-1} \leq 1+2 \varepsilon
$$
and then
$$
0 \leq \frac{\hat{\mathbf{p}}_{c}+\frac{\varepsilon\left(1+\mathbf{z}_{c}(x)\right)}{m-1}}{1+2 \varepsilon} \leq 1 .
$$
Thus, $\hat{\mathbf{q}}$ satisfies $0 \leq \hat{\mathbf{y}}_i\leq 1$.

To prove $\sum_{i=1}^{m} \hat{\mathbf{y}}_i = 1$, we use the fact that $\mathbf{z}_{c^{*}}+$ $\mathbf{z}_{i \neq c^{*}}=0$ and obtain

\begin{align*}
\sum_{i=1}^{c} \hat{\mathbf{y}}_i&=\frac{\hat{\mathbf{p}}_{c^{*}}+\varepsilon\left(1+\mathbf{z}_{c^{*}}\right)}{1+2 \varepsilon}+\sum_{i \neq c^{*}} \frac{\hat{\mathbf{p}}_{i}+\frac{\varepsilon\left(1+\mathbf{z}_{i}\right)}{m-1}}{1+2 \varepsilon} \\
&=\sum_{i=1}^{m} \frac{\hat{\mathbf{p}}_{i}}{1+2 \varepsilon}+\sum_{i \neq c^{*}} \frac{\varepsilon\left(1+\mathbf{z}_{c^{*}}+1+\mathbf{z}_{i}\right)}{(m-1)(1+2 \varepsilon)} \\
&=\frac{1}{1+2 \varepsilon}+\frac{2 \varepsilon}{1+2 \varepsilon} \\
&=1
\end{align*}
\end{proof}
\end{lemma}

\subsection{Experiment Details} \label{sec:details}
We provide more details for the experiments in this section. 

We build our classification models upon \texttt{bert-base-uncased} from Hugging Face\footnote{https://huggingface.co/}. The model contains 110M parameters. We add a dropout layer before the last linear layer with a dropout rate of 0.5. We implement \method in PyTorch 1.11.0 on a server with 4 NVIDIA TITAN-Xp GPUs. We set batch size to 8 for SST-2 and MRPC tasks, and 32 for POS Tagging and NER tasks. 

We train the victim model using AdamW \cite{Loshchilov2019DecoupledWD} optimizer with learning rate 1e-5 and epsilon 1e-8. Each victim model is trained 40 epochs and the one with the best validation results is chosen. 

Regarding the extracted model, we use half of the training data to query the victim model and obtain the labeled dataset. Then the extracted model is trained with Adam \cite{Kingma2015AdamAM} optimizer for 20 epochs with learning rate 5e-5. The average training time is 3 minutes for each epoch.

We show the results for RoBERTa model in Section \ref{sec:casestudy}. In this setting, we choose \texttt{roberta-base} from Hugging Face, which has 125M parameters.

\end{document}